\documentclass[12pt,pdftex,noinfoline]{imsart}

\RequirePackage[OT1]{fontenc}
\usepackage{mymathstyle}
\usepackage{papercommands}

\usepackage{pdfcomment}
\usepackage[draft, margin, noinline, mode=multiuser, author=]{fixme}
\usepackage[us]{datetime}
\fxloadlayouts{pdfcnote}
\fxuseenvlayout{colorsig} 
\fxusetargetlayout{changebar} 
\FXRegisterAuthor{aa}{anaa}{AA}
\FXRegisterAuthor{jl}{anjl}{JL}
\fxusetheme{colorsig}
\fxsetface{margin}{\scriptsize}

\usepackage[dvipsnames]{xcolor}
\usepackage{hyperref}

\colorlet{mylinkcolor}{RubineRed}
\colorlet{mycitecolor}{Emerald}
\colorlet{myurlcolor}{RoyalBlue}

\hypersetup{
  linkcolor  = mylinkcolor,
  citecolor  = mycitecolor,
  urlcolor   = myurlcolor,
  colorlinks = true,
}

\makeatletter
\let\c@author\relax
\makeatother
\usepackage[
    style=authoryear-comp,
    dashed=false,
    sorting=nty,
    sortcites=false,
    maxbibnames=8,
    doi=false,
    url=false,
    isbn=false,
    eprint=false,
    backend=bibtex,
    natbib
    ]{biblatex}
\AtEveryBibitem{\clearfield{month}}
\AtEveryBibitem{\clearfield{day}}
\renewbibmacro*{volume+number+eid}{%
  \printfield{volume}%
  \setunit*{\addnbspace}%
  \printfield{number}%
  \setunit{\addcomma\space}%
  \printfield{eid}}
\DeclareFieldFormat[article]{number}{\mkbibparens{#1}}
\usepackage{xpatch}
\xpatchbibdriver{article}
  {\usebibmacro{title}%
   \newunit}
  {\usebibmacro{title}%
   \printunit{\addcomma\space}}
  {}
  {}
\xpatchbibdriver{inproceedings}
  {\usebibmacro{title}%
   \newunit}
  {\usebibmacro{title}%
   \printunit{\addcomma\space}}
  {}
  {}
\renewbibmacro{in:}{}
\def\coloneq{:=}

\begin{document}
\def\snote#1{${}^{#1}$}
\setlength{\parskip}{0.5em}
\begin{frontmatter}
{\bf\Large Approximation of relation functions and attention mechanisms}
\begin{aug}
\vskip15pt
\address{
\begin{tabular}{ccccc}
{\normalsize\rm\bfseries Awni Altabaa}\snote{1} & {\normalsize\rm\bfseries John Lafferty}\snote{2}\\[5pt]
\end{tabular}
\vskip5pt
\footnotetext{
\snote{1}Department of Statistics and Data Science, Yale University; awni.altabaa@yale.edu.
\snote{2}Department of Statistics and Data Science, Wu Tsai Institute, Institute for Foundations of Data Science, Yale University; john.lafferty@yale.edu.
}
\today
\vskip10pt
}
\begin{abstract}
    Inner products of neural network feature maps arise in a wide variety of machine learning frameworks as a method of modeling relations between inputs.
    This work studies the approximation properties of inner products of neural networks.
    It is shown that the inner product of a multi-layer perceptron with itself is a universal approximator for symmetric positive-definite relation functions. In the case of asymmetric relation functions, it is shown that the inner product of two different multi-layer perceptrons is a universal approximator.
    In both cases, a bound is obtained on the number of neurons required to achieve a given accuracy of approximation. In the symmetric case, the function class can be identified with kernels of reproducing kernel Hilbert spaces, whereas in the asymmetric case the function class can be identified with kernels of reproducing kernel Banach spaces.
    Finally, these approximation results are applied to analyzing the attention mechanism underlying Transformers, showing that any retrieval mechanism defined by an abstract preorder can be approximated by attention through its inner product relations.
    This result uses the Debreu representation theorem in economics to represent preference relations in terms of utility functions.
\end{abstract}

\end{aug}
\end{frontmatter}

\section{Introduction}\label{sec:intro}

In this paper, we study the ability of neural networks to approximate relation functions, which lie at the heart of many machine learning frameworks. In general, machine learning systems must be able to represent and reason about relations between objects, either explicitly or implicitly. For example, a natural language understanding system takes a sequence of words as input and extracts information about the meaning and function of the words based on relations between them in the local context. Similarly, a scene analysis system considers the relations between the components of a scene in order to identify and interpret the implicit objects.

A common way to represent relations between objects is through inner products between feature representations of the form $\iprod{\phi(x)}{\psi(y)}$, where $x, y \in \calX$ are two objects and $\phi, \psi$ are neural network feature maps. Inner products possess properties that make them useful measures of similarity. The aim of this paper is to understand the representational power of this model by characterizing the class of relation functions $r: \calX \times \calX \to \reals$ that can be represented as inner products of neural networks.

The use of inner products between feature maps is widespread in machine learning architectures. 
For example, Siamese networks consist of two identical copies of a neural network, with shared parameters~\parencite{rumelhartLearningRepresentationsBackpropagating1986,langTimedelayNeuralNetwork1988,bromleySignatureVerificationUsing1993,baldiNeuralNetworksFingerprint1993,chopraLearningSimilarityMetric2005,kochSiameseNeuralNetworks2015}. Each network independently processes inputs to produce feature vectors that are then compared using some distance metric, determining the similarity or dissimilarity between the inputs. If the distance is the Euclidean distance, then $\twonorm{\phi(x) - \phi(y)}^2 = \iprod{\phi(x)}{\phi(x)} + \iprod{\phi(y)}{\phi(y)} - 2 \iprod{\phi(x)}{\phi(y)}$, where the inner product of neural networks arises.

A broad class of recent examples comes from the attention mechanism underlying Transformer models for sequence modeling~\parencite{vaswani2017attention}. In the Transformer, self-attention is implemented as
\begin{equation*}
    \begin{split}
        \alpha_{ij} &= \mathrm{Softmax}\paren{\bra{{\iprod{\phi_q(x_i)}{\phi_k(x_j)}}}_{j \in [n]}}_j\\
        x_i &\gets \sum_{j=1}^{n} \alpha_{ij} \phi_v(x_j)
    \end{split}
\end{equation*}
where $\phi_q, \phi_k$, and $\phi_v$ are learned transformations and $\iprod{\phi_q(x_i)}{\phi_k(x_j)}$ represents a relation between $x_i$ and $x_j$, determining how much $i$ ``should attend to'' $j$. A similar mechanism is at play in earlier neural architectures that implement content-addressable external memory~\parencite{gravesNeuralTuringMachines2014,gravesHybridComputingUsing2016a,pritzelNeuralEpisodicControl2017}, with the read/write operations typically being implemented using an inner product-based similarity computation followed by a softmax normalization. Attention has also been used in other architectures to model relations between different entities~\parencite{velickovicGraphAttentionNetworks2017a,santoroRelationalRecurrentNeural2018,zambaldiDeepReinforcementLearning2018a,locatelloObjectCentricLearningSlot2020b}. For example, \citet{santoroRelationalRecurrentNeural2018} propose a recurrent neural network with a memory module that employs dot product attention to allow elements of the memory store to interact.

While the Transformer models relations implicitly through its attention mechanism, inner products of features are also central to many explicitly relational neural architectures~\parencite[e.g.,][]{webbEmergentSymbols2021,kergNeuralArchitecture2022,altabaaRelationalConvolutionalNetworks2023,altabaaAbstractorsRelationalCrossattention2024,altabaa2024disentangling}. For example, in the model proposed by~\cite{kergNeuralArchitecture2022}, a similarity matrix is computed consisting of symmetric inner products between each pair of objects, $R_{i,\cdot} = \mathrm{Softmax}\pparen{\bra{\iprod{\phi(x_i)}{\phi(x_j)}}_{j\in[n]}}$.~\citet{altabaaAbstractorsRelationalCrossattention2024} propose a Transformer-based architecture imbued with relational inductive biases by replacing the values $\phi_v(x_i)$ with vector representations which identify objects but are independent of object-level features.~\citet{altabaaRelationalConvolutionalNetworks2023} propose a relational architecture where the central operation is a type of graph-structured convolution operating on a tensor of relations computed via inner products of feature maps.

In this paper we characterize the function class of inner products of neural networks, showing that inner products of neural networks are universal approximators for relation functions. Our analysis builds on and extends classical universal approximation results for neural networks \parencite[e.g.,][]{,cybenkoApproximationSuperpositions1989,barronUniversalApproximation1993}, and is related to more recent analysis of \citet{bachBreakingCurseDimensionality2016}. The following is a high-level overview of the results.

\noindent{\it\bfseries Symmetric relations.}  If the relation function $r$ defines a symmetric, positive-definite kernel, we show that the relation function can be approximated by inner products of neural networks as $$ \| r(x,y) - \langle \varphi(x), \varphi(y)\rangle \| \leq \varepsilon$$
where the neural network (MLP) $\varphi$ has $N$ neurons with $$ N \leq  d_r(\varepsilon/2) \, {\cal N}_\calF\left(\frac{\varepsilon}{4C(r) d_r(\varepsilon/2)}\right)$$
where $d_r$ is a measure of the spectral decay of the kernel, $C(r)$ is a bound on the eigenfunctions, and $\calN_\calF$ is a measure of the universal approximation properties of a given neural network class. The underlying function space can be identified with kernels of reproducing kernel Hilbert spaces.

\noindent{\it\bfseries Asymmetric relations.} This is the setting of many applications of relational learning. In this case, the relation function can be approximated in terms of two different multi-layer perceptrons $\varphi$ and $\psi$, with approximation error 
\[\| r(x,y) - \langle \varphi(x), \psi(y)\rangle \| \leq \varepsilon\]
achieved by a pair of neural networks with a number of neurons
\[N = \calO(\delta(\epsilon)^{-2 d}),\]
where $\delta(\epsilon)$ measures the continuity of $r$ and $d$ is the dimension. The underlying function space can be identified with kernels of reproducing kernel Banach spaces.

\noindent{\it\bfseries Attention mechanisms.} Here we use universal approximation results for inner products of neural networks to show that attention mechanisms, as used in Transformers, can always retrieve the ``most relevant'' element of a set of values.  More specifically, let any query $q$ define a preorder $\preceq_q$ on the input space $\calX$ such that $x \succeq_q y$ when $x$ is more relevant to $q$ than $y$. We write $\mathrm{Select}(q, \{x_i\})$ to denote the value $x_{i^*}$ such that $x_{i^*} \succ_q x_j$ for all $j \neq i^*$. That is, the function that selects the ``maximally relevant'' element in the context. We show that there exist MLPs that define an attention mechanism satisfying 
\[\|\mathrm{Select}(q, \{x_i\}_{i \in [n]}) - \mathrm{Attention}(q, \{x_i\}_{i \in [n]})\| \leq \varepsilon\]
for any $\varepsilon$. Our formulation of this result uses the Debreu representation theorem in economics to represent preference relations in terms of utility functions.
\section{Function class of symmetric inner product relations}\label{sec:symmetric_relations}

Consider a symmetric relation $r: \calX \times \calX \to \reals$ modeled as the inner product of neural network encodings of a pair of inputs, $x, y \in \calX$,
\begin{equation}\label{eq:symmetric_iprod_relation}
    r(x, y) = \iprod{\phi(x)}{\phi(y)},
\end{equation}
where $\phi$ is a neural network. Symmetric relation functions modeled in this way are natural \textit{measures of similarity}. Intuitively, $\phi$ extracts features of the objects $x,y$, and the inner product computes the similarity between those features. To see this more formally, suppose we have a normalized inner product relation (i.e., $\phi: \calX \to \bbS^{d}$), then we have
\begin{equation}
    \twonorm{\phi(x) - \phi(y)}^2 = 2 - 2 r(x, y).
\end{equation}


Thus, $r(x,y)$ is large and close to $1$ when $x$ and $y$ are similar (with respect to $\phi$), and close to $0$ when $x, y$ are dissimilar. Moreover, $r: \calX \times \calX \to \reals$ induces a geometry on $\calX$ via the pseudometric $d(x,y) = \sqrt{2 - 2 r(x,y)}$. The triangle inequality of the pseudometric corresponds to a notion of transitivity---if $x$ is related to $y$ and $y$ is related to $z$, then $x$ must be related to $z$.

In this section, we characterize the class of relation functions that can be modeled by symmetric inner products of neural networks. We show that any continuous symmetric positive definite kernel (i.e., Mercer kernel) can be approximated by a symmetric inner product of neural networks. This characterization is strict since $r(x,y) = \iprod{\phi(x)}{\phi(y)}$ is a Mercer kernel for any $\phi$ because $\forall g \in L^2(\calX)$ we have,
\begin{equation*}
    \begin{split}
        \int \int g(x) r(x,y) g(y) dx dy &= \int \int g(x) \iprod{\phi(x)}{\phi(y)} g(y) dx dy \\
        &= \iprod{\int g(x) \phi(x) dx}{\int g(y) \phi(y) dy} \\
        &= \twonorm{\int g(x) \phi(x) dx}^2 \geq 0.
    \end{split}
\end{equation*}

The main result in this section states that if $\phi$ is a universal approximator such as a multi-layer perceptron, then for any positive-definite symmetric kernel $r$, there exists a neural network $\phi$ such that the induced inner product relation approximates $r$ arbitrarily well. To state the result, we begin with some preliminaries.

\textbf{Preliminaries: symmetric positive definite kernels.} Consider a symmetric positive definite kernel $K: \calX \times \calX \to \reals$ on a compact euclidean space $\calX$. By Mercer's theorem~\parencite{mercerFunctionsPositive1909, sunMercerTheorem2005, micchelliUniversalKernels2006}, there exist non-negative eigenvalues $\set{\lambda_i}_{i=1}^{\infty}$ and continuous eigenfunctions $\set{\psi_i}_{i=1}^{\infty}$ such that for any $x,y \in \calX$, $K(x, y) = \sum_{i=1}^\infty \lambda_i \psi_i(x) \psi_i(y)$, and the convergence of the series is absolute and uniform over $\calX$. The following definition characterizes, in a sense, the complexity of a symmetric positive-definite kernel in terms of its spectrum.

\begin{assumption}[Kernel spectrum decay]\label{ass:sym_pd_ker_specturm_decay}
	Let $K: \calX \times \calX \to \reals$ be a symmetric positive-definite kernel, with eigenvalues $\{\lambda_i\}_{i=1}^{\infty}$ and continuous eigenfunctions $\{\psi_i\}_{i=1}^{\infty}$. For $\epsilon > 0$, denote by $d_K(\epsilon) \in \naturals$ the minimum integer such that,
	\begin{equation*}
		\sup_{x, y \in \calX} \abs{K(x, y) - \sum_{i=1}^{d_K(\epsilon)} \lambda_i \psi_i(x) \psi_i(y)} \leq \epsilon.
	\end{equation*}
	Moreover, let $C(K, d) \coloneq \max_{i \in [d]} \max_{x \in \calX} \abs{\sqrt{\lambda_i} \psi_i(x)}$.
\end{assumption}

The quantity $d_K(\epsilon)$ associated with a kernel $K$ describes the number of eigenfunctions needed to approximate $K$ up to an error of $\epsilon$. For any symmetric positive-definite kernel $K$, $d_K(\epsilon)$ is finite for any $\epsilon > 0$. 
Note that $C(K, d)$ is finite for any $d$ since $\calX$ is assumed compact and the eigenfunctions $\psi_i$ are continuous. 

\textbf{Preliminaries: universal approximation of neural networks.} Let $\calV_N$ be a class of neural networks with ``complexity'' $N$. In this work, we take $N$ to be the total number of neurons. We assume that $\calV_{N+1} \supset \calV_N$, and denote $\calV = \cup_N{\calV_N}$. Let $\calF$ be a space of functions from $\calX$ to $\reals$ and let $\nnorm{\cdot}_\calF$ be a norm on functions from $\calX$ to $\reals$. For example, $\nnorm{\cdot}_\calF$ may be the sup-norm or an $L^p$ norm, and $\calF$ may be a class of functions of a particular smoothness. For multivariate functions from $\calX^n$ to $\reals$, we denote the corresponding norm by $\nnorm{\cdot}_{\calF^{\otimes n}}$. In the case of the sup-norm, it is $\sup_{\bm{x} \in \calX^n} \abs{f(x_1, \ldots, x_n)}$ and in the case of a $L^p$ norm it is $\pparen{\int_{\calX^n} \abs{f(x_1, \ldots, x_n)}^p d\mu^{\otimes n}(x_1, \ldots, x_n)}^{1/p}$. We will consider approximating the relation function $r: \calX \times \calX \to \reals$ with respect to $\nnorm{\cdot}_{\calF^{\otimes 2}}$.

Universal approximation results provide theoretical guarantees about how well a class of networks of a given complexity can approximate a given function with respect to $\nnorm{\cdot}_\calF$. In particular, for a given function $f$, universal approximation results characterize the quantity
\[\mathrm{dist}(f, \calV_N) = \inf_{g \in \calV_N} \norm{f - g}_{\calF}.\]
Alternatively, this can be stated as follows: for a desired approximation error $\epsilon$, how large does $N$ need to be so that there exists $g \in \calV_N$ such that $\norm{f - g}_{\calF} \leq \epsilon$? We state this below as an assumption on the network class $\calV$.

\begin{assumption}[Efficiency of function approximator]\label{ass:univ_approx_efficiency}
	Consider a class of neural networks $\calV = \cup_N \calV_N$ on a space $\calX$ where $N$ is a scalable complexity. Let $\calF$ be a class of functions on $\calX$. For $\epsilon > 0$, denote by $\calN_\calF(\epsilon) \in \naturals$ the minimum integer such that for any function $f \in \calF$, there exists a neural network $g \in \calV_N$ with $N = \calN_\calF(\epsilon)$ that approximates $f$ with respect to $\norm{\cdot}_\calF$ with error bounded by $\epsilon$. That is,
	\begin{equation*}
		\sup_{f \in \calF} \, \inf_{g \in \calV_N} \norm{f - g}_{\calF} \leq \epsilon.
	\end{equation*}
	We assume a universal approximation property on the class of neural networks $\calV$: for every $\epsilon > 0$, $\calN_\calF(\epsilon)$ is finite.
\end{assumption}

We are now ready to formally state the main result of this section.

\begin{theorem}[Function class of symmetric inner product relational neural networks]\label{theorem:symmetric_inner_prod_rels_func_class}
	Suppose the data lies in a compact Euclidean space $\mathcal{X}$. Let $\calF$ be a function space from $\calX$ to $\reals$, and $\nnorm{\cdot}_\calF$ a norm on functions from  $\calX$ to $\reals$.
	Let $r: \calX \times \calX$ be a symmetric positive-definite kernel satisfying~\Cref{ass:sym_pd_ker_specturm_decay} and suppose $\sqrt{\lambda_i} \psi_i \in \calF, i \in [d_r(\epsilon / 2)]$, where $\lambda_i, \psi_i$ are the eigenvalues and eigenfunctions of $r$.
	Let $\calV = \sset{\calV_N}_N$ be a family of neural networks satisfying~\Cref{ass:univ_approx_efficiency}.
	Then for any $\epsilon > 0$, there exists a neural network $\phi \in \calV_N$ such that 
	\[\norm{r(x,y) - \iprod{\phi(x)}{\phi(y)}}_{\calF^{\otimes 2}} < \epsilon.\]
	Moreover, the complexity of the neural network required is bounded by
	\[d_r(\epsilon/2) \cdot \calN_\calF\paren{\frac{\epsilon}{4 C(r) d_r(\epsilon/2)}}.\]
\end{theorem}

\begin{proof}
	By Mercer's theorem~\parencite{mercerFunctionsPositive1909, sunMercerTheorem2005, micchelliUniversalKernels2006}, there exists $(\psi_i)_{i \in \mathbb{N}}$, $\lambda_i \geq 0$ such that $r(x,y) = \sum_{i=1}^{\infty}{\lambda_i \psi_i(x) \psi_i(y)}$, where $\psi_i$ and $\lambda_i$ are eigenfunctions and eigenvalues of the integral operator $T_r: L^2(\mathcal{X}) \to L^2(\mathcal{X}),\,T_r: f \mapsto \int_{\mathcal{X}}{r(\cdot, x) f(x) dx}$.
	Furthermore, the convergence of the series is uniform:
	\begin{equation}
		\lim_{n \to \infty} \sup_{x,y \in \mathcal{X}} \lvert r(x,y) - \sum_{i=1}^{n}{\lambda_i \psi_i(x) \psi_i(y) \rvert} = 0.
	\end{equation}
	Using the notation of~\Cref{ass:sym_pd_ker_specturm_decay}, we have that
	\begin{equation}\label{eq:proof_mercer_thm_unif_abs_cv}
		\sup_{x,y \in \mathcal{X}} \left\lvert r(x,y) - \sum_{i=1}^{d_r(\epsilon/2)}{\lambda_i \psi_i(x) \psi_i(y)} \right\rvert < \frac{\epsilon}{2},
	\end{equation}
	where $d_r(\epsilon/2)$ is defined as in~\Cref{ass:sym_pd_ker_specturm_decay}.

	Let $d \coloneqq d_r(\epsilon / 2)$ be the output dimension of the neural network $\phi$ such that it maps from $\calX$ to $\reals^d$. Our approach will be to approximate with the MLP $\phi$ the first $d$ eigenfunctions of the kernel relation $r$, $(\sqrt{\lambda_1} \psi_1, \ldots, \sqrt{\lambda_{d}} \psi_{d})$.
	By the universal approximation property of $\calV$ (\Cref{ass:univ_approx_efficiency}), for any $\epsilon_1 > 0$, there exists a neural network $\phi = \pparen{\phi_1, \ldots, \phi_d} \in \calV_N$ for some $N$ such that
	\begin{equation}\label{eq:proof_NN_UAP}
		\norm{\phi_i - \sqrt{\lambda_i} \psi_i}_\calF < \epsilon_1, \ \forall i \in \{1, \ldots, d\},
	\end{equation}
	where $\phi_i(x)$ is the $i$-th component of $\phi(x)$. Moreover, by~\Cref{ass:univ_approx_efficiency} there exists such a $\phi$ with a number of neurons bounded by $d \cdot \calN_\calF(\epsilon_1)$.

	Now note that the approximation error for $r$ is bounded by
	\begin{equation}\label{eq:proof_approx_bound}
		\begin{split}
			&\bignorm{r(x,y) - \langle \phi_\theta(x), \phi_\theta(y) \rangle}_{\calF^{\otimes 2}}\\
			&= \norm{ r(x,y) - \sum_{i=1}^{d}{\phi_i(x) \phi_i(y)}}_{\calF^{\otimes 2}} \\
			&\leq \norm{ r(x,y) - \sum_{i=1}^{d}{\lambda_i \psi_i(x) \psi_i(y)}}_{\calF^{\otimes 2}} + \norm{\sum_{i=1}^{d}\lambda_i \psi_i(x) \psi_i(y) - \sum_{i=1}^{d}\phi_i(x) \phi_i(y)}_{\calF^{\otimes 2}}
		\end{split}
	\end{equation}
	The first term is bounded by $\frac{\epsilon}{2}$ by~\eqref{eq:proof_mercer_thm_unif_abs_cv}. The second term can be bounded by noting the for any $x, y \in \calX$
	\begin{equation*}
		\begin{split}
			&\abs{\sum_{i=1}^{d}\lambda_i \psi_i(x) \psi_i(y) - \sum_{i=1}^{d}\phi_i(x) \phi_i(y)} \\
			&\leq \sum_{i=1}^{d}{ \abs{ \lambda_i \psi_i(x) \psi_i(y) - \phi_i(x) \phi_i(y)}} \\
			&\leq \sum_{i=1}^{d}{\paren{
				\abs{\sqrt{\lambda_i} \psi_i(y)} \abs{\sqrt{\lambda_i} \psi_i(y) - \phi_i(y)}
				+ \abs{\sqrt{\lambda_i} \psi_i(y)} \abs{\sqrt{\lambda_i} \psi_i(x) - \phi_i(x)}
				}} \\
			&\leq C(r) \sum_{i=1}^{d}{\paren{
				\abs{\sqrt{\lambda_i} \psi_i(y) - \phi_i(y)}
				+ \abs{\sqrt{\lambda_i} \psi_i(x) - \phi_i(x)}
				}} \\
		\end{split}
	\end{equation*}
	where the last line is by the definition $C(r) \coloneqq \max_{x \in \calX} \lvert \sqrt{\lambda_i} \psi_i(x) \rvert$ (\Cref{ass:sym_pd_ker_specturm_decay}). Now, by~\Cref{eq:proof_NN_UAP}, we have
	\begin{align*}
		&\norm{\sum_{i=1}^{d}\lambda_i \psi_i(x) \psi_i(y) - \sum_{i=1}^{d}\phi_i(x) \phi_i(y)}_{\calF^{\otimes 2}} \\
		&\leq \norm{C(r) \sum_{i=1}^{d}{\paren{\abs{\sqrt{\lambda_i} \psi_i(y) - \phi_i(y)} + \abs{\sqrt{\lambda_i} \psi_i(x) - \phi_i(x)}}}}_{\calF^{\otimes 2}}\\
		&\leq C(r) \sum_{i=1}^{d}\paren{\norm{\sqrt{\lambda_i} \psi_i(y) - \phi_i(y)}_{\calF^{\otimes 2}} + \norm{\sqrt{\lambda_i} \psi_i(x) - \phi_i(x)}_{\calF^{\otimes 2}}}\\
		&\leq 2 C(r) \cdot d \cdot \epsilon_1,
	\end{align*}

	Let the neural network approximation error be $\epsilon_1 = \frac{\epsilon}{4 C(r) d_r(\epsilon / 2)}$ such that the above is bounded by $\epsilon / 2$. 

	Then, by~\eqref{eq:proof_approx_bound}, we have that
	\begin{equation*}
			\bignorm{r(x,y) - \iprod{\phi(x)}{\phi(y)}}_{\calF^{\otimes 2}} \leq \frac{\epsilon}{2} + \frac{\epsilon}{2} = \epsilon.
	\end{equation*}

	Hence, the relation function $r$ is approximated with respect to $\nnorm{\cdot}_{\calF^{\otimes 2}}$ by an inner product of neural networks with number of neurons at most 
	\[d_r(\epsilon / 2) \cdot \calN_\calF\paren{\frac{\epsilon}{4 C(r) d_r(\epsilon / 2)}}.\]

\end{proof}

\Cref{theorem:symmetric_inner_prod_rels_func_class} states that pairwise relations modeled as inner products of neural networks can capture any symmetric positive definite kernel. Moreover, the scale of the neural network needed to achieve a particular approximation error is characterized in terms of the complexity of the kernel relation function and the efficiency of the neural network function class. In particular, this dependence is expressed in the kernel relation's spectrum decay $d_r(\cdot)$ and the neural network's efficiency $\calN_\calF(\cdot)$.

\begin{remark}
	In the results above, the number of neurons is bounded by ``$d_r \cdot \calN$''. This is an upper bound assuming each of the $d_r$ kernel eigenfunctions is modeled independently. In practice, the size of the neural network needed would be smaller since the eigenfunctions can be approximated with a single MLP with a $d_r$-dimensional output, allowing for computations to be re-used across several output dimensions.
\end{remark}

The efficiency of neural networks in approximating arbitrary functions has been studied extensively. Different results in the literature characterize $\calN_{\calF}(\cdot)$ as a function of the class of neural networks (e.g., shallow or deep) and the function class $\calF$ (e.g., degree of smoothness). In the following corollary, we specialize~\Cref{theorem:symmetric_inner_prod_rels_func_class} to approximation with respect to the sup-norm $\nnorm{\cdot}_\calF = \nnorm{\cdot}_\infty$ and $\calF$ the class of Lipschitz functions.

\begin{corollary}\label{cor:sym_iprod_kernel_neuron_bound}
	Consider the setting of~\Cref{theorem:symmetric_inner_prod_rels_func_class}. Suppose further that $\calX = [-1, 1]^d$ and let $L(r, k) \coloneq \max_{i \in [k]} L_i$, where $L_i$ is the Lipschitz constant of $\sqrt{\lambda_i} \psi_i$ and $\lambda_i, \psi_i$ are the eigenvalues and eigenfunctions of the symmetric relation function $r$. Consider $\calV$ to be the class of shallow neural networks with ReLU activations with $\phi \in \calV$ given by
	\begin{equation*}
		\phi(x) = \sum_{k=1}^{N} a_k \mathrm{ReLU}(\iprod{w_k}{x} + b_k),
	\end{equation*}
	where $a_k, b_k \in \reals$, $w_k \in \reals^{d}$ are the parameters, and $N$ is the number of neurons. Then, for any $\epsilon > 0$, there exists $\phi \in \calV_N$ achieving $\sup_{x,y \in \mathcal{X}}{\abs{r(x,y) - \iprod{\phi(x)}{\phi(y)}}} < \epsilon$ with a number of neurons $N$ of the order
	\[\calO\paren{d_r(\epsilon / 2) \cdot \paren{\frac{4 C(r) d_r(\epsilon / 2) L(r, d_r(\epsilon/2))}{\epsilon}}^{d}}.\]
\end{corollary}
The proof of this result follows from~\Cref{theorem:symmetric_inner_prod_rels_func_class} and~\textcite{bachBreakingCurseDimensionality2016}.

We observe a curse of dimensionality in this approximation result, where the number of neurons needed to achieve a good approximation grows exponentially in the dimension $d$ of the underlying space $\calX$. This phenomenon has been studied in the context of neural network approximation results, and it's well knwon that imposing further structural assumptions on the functions to be approximated can reduce or remove the dependence on dimension. One approach to obtaining a more favorable dependence on the dimension was outlined by~\textcite{barronUniversalApproximation1993}'s seminal work, which identifies a smoothness condition based on the Fourier representation of the underlying function.  The \textit{Barron norm} of a function $f: \reals^d \to \reals$ is defined as
\[\nnorm{f}_{\calB} \coloneq \int_{\reals^d} \nnorm{\hat{\nabla f}(\omega)} d\omega = \int_{\reals^d} \nnorm{\omega} \aabs{\hat{f}(\omega)} d \omega,\]
where $\hat{f}$ denotes the Fourier transform. The advantage of this measure of smoothness is that, for some interesting classes of functions, it scales sub-exponentially with the dimension. In the following corollary, we show that a more favorable dependence on the dimension can be obtained if the eigenfunctions of the relation are smooth in the sense of the Barron norm.

\begin{corollary}\label{cor:sym_iprod_kernel_barron_neuron_bound}
	Consider the setting of~\Cref{theorem:symmetric_inner_prod_rels_func_class}. Suppose that $\calX$ has radius $\mathrm{radius}(\calX)$ (i.e., $\nnorm{x} \leq \mathrm{radius}(\calX),\, \forall x \in \calX$). Assume that $\nnorm{\sqrt{\lambda_i} \psi_i}_{\calB} \leq B(r)$ for $i \in [d_r(\epsilon/2)]$. Consider $\calV$ to be the class of shallow neural networks with sigmoid activations with $\phi \in \calV$ given by
	\begin{equation*}
		\phi(x) = \sum_{k=1}^{N} a_k \sigma(\iprod{w_k}{x} + b_k),
	\end{equation*}
	where $a_k, b_k \in \reals$, $w_k \in \reals^{d}$ are the parameters, and $N$ is the number of neurons. Then, for any $\epsilon > 0$, there exists $\phi \in \calV_N$ achieving $\nnorm{r(x,y) - \iprod{\phi(x)}{\phi(y)}}_{L^2} < \epsilon$ with a number of neurons $N$ of the order
	\[\calO\paren{d_r(\epsilon / 2)^3 \cdot \frac{\paren{C(r) \cdot \mathrm{radius}(\calX)\, \cdot B(r)}^2}{\epsilon^2}}.\]
\end{corollary}
The proof of this statement follows from~\Cref{theorem:symmetric_inner_prod_rels_func_class} and~\textcite{barronUniversalApproximation1993}. 
This result shows that the size of the neural network needed to approximate $r$ scales with its spectrum decay $d_r(\cdot)$, which can be interpreted as a measure of the dimensionality or ``rank'' of the relation, and the smoothness of its eigenfunctions $B(r)$.

\section{Function class of asymmetric inner product relations}\label{sec:asymmetric_relations}

In the previous section, we considered symmetric inner product relations where the encoder of the first object is the same as the encoder of the second object. When the underlying relation being modeled is a symmetric `similarity' relation, this is a useful inductive bias. However, in general, relations between objects can be asymmetric. One example of an asymmetric (in fact, anti-symmetric) relation is \textit{order}. Such relations cannot be captured by symmetric inner products. In this section, we consider modeling a general asymmetric relation $r: \calX \times \calX \to \reals$ as the inner product of two different neural network encodings of a pair of objects,
\begin{equation}\label{eq:asymmetric_iprod_relation}
    r(x, y) = \iprod{\phi(x)}{\psi(y)},
\end{equation}
where $\phi, \psi: \calX \to \reals^d$ are two neural networks.
We show that inner products of multi-layer perceptrons can approximate any continuous function on $\calX \times \calX$.

We begin with the following simple lemma, which states when the object space $\calX$ is finite, any relation function can be represented as the inner product between two encodings.

\begin{lemma}\label{lemma:finite_space_rel}
    Suppose $\calX$ is a finite space. Let $r: \calX \times \calX \to \reals$ be any relation function. Then, there exists $d \leq \abs{\calX}$ and $\phi, \psi: \calX \to \reals^{d}$ such that,
    \begin{equation*}
        r(x, y) = \iprod{\phi(x)}{\psi(y)}, \ \forall x, y \in \calX.
    \end{equation*}
\end{lemma}

To see this, let $x_1, \ldots, x_n$ be an enumeration of $\calX$ where $m = \abs{\calX}$. Let $R \in \reals^{n \times n}$ be such that $R_{ij} = r(x_i, x_j)$. There exists many decompositions of the matrix $R$ which would induce valid encodings $\phi, \psi$. One example is rank decomposition. Let $d = \mathrm{rank}(R)$. Then, there exists matrices $P, Q \in \reals^{d \times n}$ such that $R = P^\top Q$. Let $\phi, \psi: \calX \to \reals^{d}$ be defined by
    \begin{equation}
        \phi(x_i) = P_{i, \cdot}, \ \psi(x_i) = Q_{\cdot, i}, \ \forall i \in [m].
    \end{equation}
Then, $r(x, y) = \iprod{\phi(x)}{\psi(y)}$ for all $x, y \in \calX$.

Note that if each $x \in \calX$ is a one-hot vector in $\reals^{\abs{\calX}}$, then the result above holds with linear maps $\phi, \psi$. Although very simple, this result has direct implications for domains such as language modeling where $\calX$ is a discrete set of tokens, and hence finite. In such cases,~\Cref{lemma:finite_space_rel} tells us that any relation function can be approximated by inner products of feature maps (i.e., of the form present in the attention mechanisms of Transformers). Moreover, in the case of language, there may be a low-rank structure (e.g., depending on syntax, semantics, etc.) enabling a more modest dimension of the feature maps, $d \ll \abs{\calX}$.

Next, we proceed to show that arbitrary continuous relation functions can be approximated by inner products of two different neural networks. Our strategy will be to first quantize the space $\calX$ and then apply the construction above for the finite case.

\begin{theorem}\label{theorem:asymemtric_inner_prod_rel_func_class}
    Suppose the relation function $r: \calX \times \calX \to \reals$ is continuous. 
    Then, for any approximation error $\epsilon > 0$ there exist multi-layer perceptrons $\phi, \psi$ such that
    \begin{equation*}
        \abs{r(x,y) - \iprod{\phi(x)}{\psi(y)}} \leq \epsilon, \quad \text{Lebesgue-almost everywhere.}
    \end{equation*}
    Moreover, $\phi, \psi$ can be constructed such that $\phi = L_\phi \circ \eta$, $\psi = L_\psi \circ \eta$ where $\eta$ is a shared 2-layer MLP with $N = \calO(\delta^{- 2 \,\dim(\calX)})$ neurons and $L_\phi, L_\psi$ are linear projections onto $n$-dimensional space, with $n = \calO(\delta^{- \dim(\calX)})$ and $\delta = \delta(\epsilon)$.
\end{theorem}
\begin{proof}
    Let $x_1, \ldots, x_n \in \calX$ be a set of points in $\calX$. Define the Voronoi partition by
    \[V^{(i)} = \sset{x \in \calX \,|\, \norm{x - x_i} \leq \norm{x - x_j} \ \forall j \neq i}.\]
    Let $x_1, \ldots, x_n$ be uniformly distributed in $\calX$ and $n = \calO(\delta(\epsilon)^{- \dim(\calX)})$ so that the maximal diameter of the sets $V^{(1)}, \ldots, V^{(n)}$ is bounded by $\delta(\epsilon)$, with $\max_{i \in [n]} \mathrm{diam}(V^{(i)}) \leq \delta$. Let $q: \calX \to \sset{x_1, \ldots, x_n}$ be the quantizer which maps each $x$ to the closest element in $\sset{x_1, \ldots, x_n}$.

    \citet{wuExplicitNeuralNetwork2018} explicitly construct a two-layer neural network $\eta:\calX \to \{0, 1\}^{n}$ such that
    \[{(\eta(x))}_i = 1 \iff x \in V^{(i)}.\]
    The construction contains $n (n - 1)$ neurons in the first layer and $n$ neurons in the second layer, both with the threshold function $\sigma(x) = \bm{1}\sset{x \geq 0}$ as the activation function. Note that sigmoidal activation functions can approximate the step function arbitrarily well. The weights between the first and second layers are sparse. The neural network takes the form,
    \begin{equation*}
        \begin{aligned}
            z_{k,j}^{(1)}(x) &= \sigma\paren{w_{k,j}^{(1)} \cdot x - b_{k,j}^{(1)}}, \quad k,j \in [n], k \neq j\\
            z_k^{(2)}(x) &= \sigma\paren{w_k^{(1)} \cdot \bm{z}^{(1)}(x) - b^{(2)}}, \quad k \in [n] \\
            \eta(x) &= \bm{z}^{(2)}(x) = \paren{z_1^{(2)}(x), \ldots, z_n^{(2)}(x)},
        \end{aligned}
    \end{equation*}
    where the weights are $w_{k,j}^{(1)} = x_k - x_j, b_{k,j}^{(1)} = \frac{1}{2} \iiprod{x_k - x_j}{x_k + x_j}, (w_{k}^{(2)})_{a,b} = \bm{1}\sset{a = k}, b^{(2)} = n - 1$, and $\bm{z}^{(1)} = (z_{k,j}^{(1)})_{k \neq j} \in \reals^{n (n - 1)}$ are the first layer activations.

    Let $R \in \reals^{n \times n}$ be defined by $\bbra{R}_{ij} = r(x_i, x_j)$. The matrix $R$ specifies the relation function on the sample points $x_1, \ldots, x_n$. We have that $R = P^\top Q,\, P,Q \in \reals^{m \times n}$ for some $m \leq n$. 

    Let $\phi = P \circ \eta$ and $\psi = Q \circ \eta$. To show that the inner product of neural networks $\iprod{\phi(x)}{\psi(y)}$ approximates $r(x, y)$, take $x, y \in \calX$. We have
    \begin{align*}
        \abs{r(x, y) - \iprod{\phi(x)}{\psi(y)}} &\leq \abs{r(x, y) - r(q(x), q(y))} + \abs{r(q(x), q(y)) - \iprod{\phi(x)}{\phi(y)}}\\
        &\leq \epsilon + \abs{r(q(x), q(y)) - \eta(x)^\top R\, \eta(y)},
    \end{align*}
    where the second inequality follows by the assumption of continuity on the relation function $r$ and the choice of the quantization diameter $\delta$. Now, note that whenever $x \in \mathrm{int}(V^{(i)})$, we have $\eta(x) = e_i$, where $e_i$ is the canonical basis vector. Hence $r(q(x), q(y)) = \eta(x)^\top R\, \eta(y)$ Lebesgue-almost everywhere. 
\end{proof}

\begin{remark}
    From a learning perspective, we only need to know the value of the relation function at $n$ (uniformly distributed) points $x_1, \ldots, x_n$, with $n = \calO(\delta^{-d})$ depending on the smoothness of $r$ and the dimension of the space $\calX$.
\end{remark}

\begin{remark}
    In the symmetric case, where the relation function $r$ is assumed to be a positive-definite kernel, Mercer's theorem implies an inner product-like structure in $r$, with a ``rank'' determined by the spectral decay of the kernel. In the asymmetric case, the above result assumes only the continuity of the relation function $r$ and no further ``inner-product-like'' structure. In some applications, the relation function may have a low-rank structure such that it is representable by an inner product between low-dimensional feature maps, enabling more favorable bounds on the size of the neural network needed, even in the asymmetric case.
\end{remark}

Here, again, we observe a curse of dimensionality in~\Cref{theorem:asymemtric_inner_prod_rel_func_class}, where the number of neurons needed to obtain a good approximation scales exponentially with the dimension of the underlying space $\calX$. We note that~\Cref{theorem:asymemtric_inner_prod_rel_func_class} makes very mild regularity conditions on the relation function to be approximated---it only needs to be continuous as a function from $\calX \times \calX$ to $\reals$. In general, an exponential dependence on $2\,\cdot\, \dim(\calX)$ is necessary for any method of approximation~\citep{pinkus1999approximation,devore1998nonlinear,maiorov1999lower,maiorov2000near,poggioWhyWhenCan2017}.

This exponential dependence on dimension can be avoided with additional structure on the relation function. Here, we consider a compositional structure where the relation depends on the object vectors only through a smooth low-dimensional feature filter. That is, the target relation function $r: \calX \times \calX \to \reals$ takes the form
\begin{equation*}
    r(x, y) = \bar{r}(\xi(x), \xi(y)),
\end{equation*}
where $\xi : \calX \to \reals^k$ is a smooth $k$-dimensional feature filter and $\bar{r}$ is a continuous relation on the feature $\xi$. For example, $\calX$ may be a high-dimensional image space, while $\xi$ is a low-dimensional filter representing a particular visual attribute such as color or texture in some patch of the image. We will consider smoothness of $\xi$ as measured by the Barron norm, $\nnorm{f}_{\calB} \coloneq \int_{\reals^d} \nnorm{\hat{\nabla f}(\omega)} d\omega$. We write $\nnorm{\xi}_{\calB}$ to mean $\max_{i \in [k]} \nnorm{\xi_i}_{\calB}$.

\begin{corollary}
    Consider the setting of~\Cref{theorem:asymemtric_inner_prod_rel_func_class}. Suppose the relation function $r: \calX \times \calX \to \reals$ has the form $r(x, y) = \bar{r}(\xi(x), \xi(y))$ where $\xi: \calX \to \reals^k$ has Barron norm $\nnorm{\xi}_\calB$ and $\bar{r}: \reals^k \times \reals^k \to \reals$ is $L$-Lipschitz in the sense that $$\aabs{\bar{r}(x, y) - \bar{r}(x', y')} \leq L \cdot (\max(\nnorm{x - x'}_\infty, \nnorm{y - y'}_\infty)), \, \forall x,x',y,y' \in \xi(\calX).$$ Then, for any $\epsilon > 0$, there exists $\hat{\xi}, \phi = L_\phi \circ \eta, \psi = L_\psi \circ \eta$, such that
    \[\norm{r(x, y) - \iprod{\phi(\hat{\xi}(x))}{\psi(\hat{\xi}(y))}}_{L^2} \leq \epsilon\]
    where $\hat{\xi}$ is a shallow 1-hidden layer neural network with $\calO(k \cdot \frac{\nnorm{\xi}_{\calB}^2}{\epsilon^2})$ hidden neurons and $k$ output neurons, $\eta$ is a 2-layer neural network with $\calO((\frac{L}{\epsilon})^{2 k})$ neurons, and $L_\phi, L_\psi$ are linear projections onto $\calO((\frac{L}{\epsilon})^{k})$-dimensional space.
\end{corollary}
\begin{proof}
    We have
    \begin{align*}
        &\norm{r(x, y) - \iprod{\phi(\hat{\xi}(x))}{\psi(\hat{\xi}(y))}}_{L^2} \\
        &\leq \norm{\bar{r}(\xi(x), \xi(y)) - \bar{r}(\hat{\xi}(x), \hat{\xi}(y))}_{L^2} + \norm{\bar{r}(\hat{\xi}(x), \hat{\xi}(y)) - \iprod{\phi(\hat{\xi}(x))}{\psi(\hat{\xi}(y))}}_{L^2} \\
        &\leq \norm{L \paren{\infnorm{\xi(x) - \hat{\xi}(x)} + \infnorm{\xi(y) - \hat{\xi}(y)}}}_{L^2} + \norm{\bar{r}(\hat{\xi}(x), \hat{\xi}(y)) - \iprod{\phi(\hat{\xi}(x))}{\psi(\hat{\xi}(y))}}_{L^2}
    \end{align*}
    The first term can be controlled by constructing $\hat{\xi}$ to approximate $\xi$ in the $L^2$ norm according to the construction of~\textcite{barronUniversalApproximation1993}. The number of neurons needed depends on the smoothness of the function $\xi$. The second term can be approximated uniformly (not just in $L^2$) by constructing MLPs $\phi, \psi$ according to~\Cref{theorem:asymemtric_inner_prod_rel_func_class} to approximate $\bar{r}$. Here, the complexity of the neural network class needed depends on $k$ rather than $d$.
\end{proof}

This result shows that if the underlying relation function possesses a compositional structure depending on a low-dimensional smooth feature map, then it can be efficiently approximated by inner products of neural networks. In particular, the number of neurons needed scales with the dimensionality $k$ of the feature filter $\xi$ and the smoothness of $\xi$, as measured by the Barron norm.
\section{Connection to reproducing kernel Hilbert and Banach spaces}\label{sec:rkbs_asymmetric_relations}

In~\Cref{sec:symmetric_relations} we showed that the function class of symmetric inner products of neural networks is the set of symmetric positive-definite kernels---that is, reproducing kernels of reproducing kernel Hilbert spaces (RKHS). There exists a similar interpretation of the function class of asymmetric inner products of neural networks in terms of the reproducing kernels of \textit{reproducing kernel Banach spaces} (RKBS).

Recall that a reproducing kernel Hilbert space $\calH$ is a Hilbert space of functions on a space $\calX$ for which the point evaluation functionals $f \mapsto f(x)$ are continuous. \citet{aronszajn1950theory} showed that there is a one-to-one identification between RKHSs and symmetric positive definite kernels $K: \calX \times \calX \to \reals$ such that $\iprod{K(x, \cdot)}{f}_\calH = f(x)$.~\citet{mercerFunctionsPositive1909} had previously shown that an RKHS can be identified with a feature map via the spectral decomposition of the integral operator $T_K: L_2(\calX) \to L_2(\calX)$ defined by $T_K f(x) = \int_\calX K(x, y) f(y) dy$. Every feature map $\phi: \calX \to \calW$ defines a symmetric positive definite kernel $K(x, y) = \iprod{\phi(x)}{\phi(y)}_\calW$ (and hence, an RKHS), while every symmetric positive definite kernel has infinitely many feature map representations.

\Cref{theorem:symmetric_inner_prod_rels_func_class} shows that the function class of symmetric inner products of neural networks is the set of reproducing kernels of RKHS function spaces. A reproducing kernel Hilbert space is, as the name suggests, a \textit{Hilbert space} of functions on some space $\calX$. The linear structure of a Hilbert space makes the kinds of geometries it can capture relatively restrictive, since any two Hilbert spaces with the same dimension are isometrically isomorphic. Banach spaces, which have fewer structural assumptions, can capture richer geometric structures. Hence, a reproducing kernel Banach space can capture richer geometries between functions than an RKHS. In particular, in contrast to an RKHS, the reproducing kernel of an RKBS need not be symmetric or positive definite. In this section, we show that the function class of asymmetric inner products of neural networks has an interpretation in terms of the reproducing kernels of RKBSs, mirroring the result for symmetric inner products of neural networks. Nonsymmetric kernels of positive type, a more restricted class, were studied by \citet{seelyNonSymmetricKernels1919}, shortly after Mercer's seminal work for the symmetric case. 

\subsection{Background on reproducing kernel Banach spaces}

\begin{definition}[Reproducing Kernel Banach Space]
    A \textbf{reproducing kernel Banach space} on a space $\calX$ is a Banach space $\calB$ of functions on $\calX$, satisfying:
    \begin{enumerate}
        \item $\calB$ is \textit{reflexive}. That is, $(\calB^*)^* = \calB$, where $\calB^*$ is the dual space of $\calB$. Furthermore, $\calB^*$ is isometric to a Banach space $\calB^\#$ of functions on $\calX$.
        \item The point evaluation functionals $f \mapsto f(x)$ are continuous on both $\calB$ and $\calB^\#$.
    \end{enumerate}
\end{definition}

This definition is a strict generalization of reproducing kernel Hilbert spaces, as any RKHS $\calH$ on $\calX$ is also an RKBS, since (1) is implied by the Riesz representation theorem. While the identification $\calB^\#$ is not unique, we can choose some identification arbitrarily and denote it by $\calB^*$ for ease of notation (by assumption, all identifications are isometric to each other). Thus, if $\calB$ is an RKBS, $\calB^*$ is also an RKBS.

Similar to an RKHS, an RKBS also has a \textit{reproducing kernel}. To state the result, for a normed vector space $\calV$ and its dual space $\calV^*$, we define the bilinear form
\begin{equation}\label{eq:bilinear_form}
    \begin{split}
        \calV \times \calV^* &\to \reals\\
        (u, v^*)_\calV &\mapsto v^*(u).
    \end{split}
\end{equation}

Theorem 2 of \citet{zhangReproducingKernel2009} shows that for any RKBS $\calB$ there exists a unique reproducing kernel $K: \calX \times \calX \to \bbC$ that recovers point evaluations, meaning
\begin{align}
    f(x) &= \paren{f, K(\cdot, x)}_\calB, \quad \forall f \in \calB, \\
    f^*(x) &= \paren{K(x, \cdot), f^*}_\calB,\quad \forall f^* \in \calB^*,
\end{align}

and such that the span of $K(x, \cdot)$ is dense in $\calB$ and the span of $K(\cdot, x)$ is dense in $\calB^*$,
\begin{align}
    \overline{\text{span}}\{K(x, \cdot): x \in \calX\} &= \calB, \\
    \overline{\text{span}}\{K(\cdot, x): x \in \calX\} &= \calB^*.
\end{align}
Finally,
\begin{equation}
    K(x, y) = \paren{K(x, \cdot), K(\cdot, y)}_\calB, \ \forall x, y \in \calX.
\end{equation}
Unlike RKHSs, while each RKBS has a unique reproducing kernel, different RKBSs may have the same reproducing kernels.

Furthermore, a kernel $K: \calX \times \calX \to \bbC$ is the reproducing kernel of some RKBS if and only if it has a feature map representation. Crucially for us, the feature map representation is more versatile than the one for RKHSs. Let $\calW$ be a reflexive Banach space with dual space $\calW^*$. Consider a pair of feature maps $\Phi$ and $\Phi^*$, mapping to each feature space, respectively. That is,
\begin{equation*}
    \Phi: \calX \to \calW, \ \Phi^*: \calX \to \calW^*,
\end{equation*}
where we call $\Phi,\, \Phi^*$ the \textit{pair} of feature maps and $\calW,\, \calW^*$ the pair of feature spaces. Suppose that the span of the image of the feature maps under $\calX$ is dense in their respective feature spaces. That is,
\begin{equation}
    \overline{\text{span}}\{\Phi(x): x \in \calX\} = \calW, \;\; \overline{\text{span}}\{\Phi^*(x): x \in \calX\} = \calW^*.
\end{equation}
Then, by Theorem 3 of \citet{zhangReproducingKernel2009}, the feature maps $\Phi, \Phi^*$ induce an RKBS defined by
\begin{align}
    \calB &:= \set{f_w: x \mapsto (\Phi^*(x))(w), w \in \calW} \\
    \norm{f_w}_\calB &:= \norm{w}_\calW,
\end{align}
with the dual space $\calB^*$ defined by
\begin{align}
    \calB^* &:= \set{f_{w^*}: x \mapsto w^*(\Phi(x)), w^* \in \calW^*} \\
    \norm{f_{w^*}}_{\calB^*} &:= \norm{w^*}_{\calW^*}.
\end{align}
Furthermore, for any RKBS, there exist some feature spaces $\calW, \calW^*$ and feature maps $\Phi, \Phi^*$ such that the above construction yields that RKBS, which is Theorem 4 of \citet{zhangReproducingKernel2009}.

\subsection{Asymmetric inner products of neural networks model kernels of reproducing kernel Banach spaces}

Observe that for an RKBS with feature-map representation given by $\Phi, \Phi^*$, its reproducing kernel is given by
\begin{equation}
    K(x, y) = \paren{\Phi(x), \Phi^*(y)}_{\calW}, \ x, y \in \calX,
\end{equation}
where $\paren{\cdot, \cdot}_{\calW}$ is the bilinear form on $\calW$ defined in~\Cref{eq:bilinear_form}.

This form is reminiscent of the asymmetric inner product of neural networks,
\begin{equation}
    r(x, y) = \iprod{\phi(x)}{\psi(y)}, \ x, y \in \calX,
\end{equation}
where $\phi, \psi: \calX \to \reals^{d}$ is a pair of learned feature maps. We will show that asymmetric inner products of neural networks can approximate any reproducing kernel of an RKBS.

Let $\Phi, \Phi^*$ be a pair of feature maps that define the reproducing kernel $r(x,y) = \pparen{\Phi(x), \Phi^*(y)}_\calW$ for some RKBS. Recall that any two Hilbert spaces with equal dimensions are isometrically isomorphic. Hence, when the feature space $\calW$ is a Hilbert space, we can consider $\calW = \ell^2(\naturals)$ without loss of generality. 

The following theorem states that when $\calB$ is an RKBS on $\calX$ with a feature map representation whose feature space $\calW$ is a Hilbert space, its reproducing kernel can be approximated by an asymmetric inner product of neural networks. The proof is similar to that of~\Cref{theorem:symmetric_inner_prod_rels_func_class}.

\begin{theorem}\label{thm:asymmetric_inner_prod_approximates_rkbs}
   Suppose $\calX$ is a compact metric space. Suppose $r: \calX \times \calX \to \reals$ is the reproducing kernel of some RKBS $\calB$ on $\calX$ admitting a continuous feature map representation with a feature space $\calW$ that is a Hilbert space. 
   Then, for any $\varepsilon > 0$, there exist multi-layer perceptrons 
    $\phi, \psi: \calX \to \reals^{d}$ such that 
   \begin{equation*}
        \sup_{x,y \in \calX} \abs{r(x, y) - \iprod{\phi(x)}{\psi(y)}} \leq \varepsilon.
   \end{equation*}
\end{theorem}

\begin{proof}
    By assumption, there exists a Hilbert space $\calW$ and a pair of feature maps $\Phi: \calX \to \calW, \Phi^*: \calX \to \calW$ such that,
    \begin{equation*}
        r(x, y) = \paren{\Phi(x), \Phi^*(y)}_{\calW} \equiv (\Phi^*(y))(x), \ x, y \in \calX.
    \end{equation*}

    Without loss of generality, we can restrict our attention to the feature space $\calW = \ell^2(\bbN)$, since any two Hilbert spaces with equal dimension are isometrically isomorphic. The dual space is $\calW^* = \ell^2(\bbN)$. Hence, for feature maps $\Phi, \Phi^*$, the ground truth relation to be approximated is,
    \begin{equation*}
        r(x, y) = \paren{\Phi(x), \Phi^*(y)}_{\ell^2(\bbN)} \equiv (\Phi^*(y))(\Phi(x)), \ x, y \in \calX.
    \end{equation*}
    By the Riesz representation theorem, there exists a unique element in $u_{\Phi^*(y)} \in \ell^2(\bbN)$ such that,
    \begin{equation*}
        (\Phi^*(y))(w) = \iprod{w}{u_{\Phi^*(y)}}_{\ell^2(\bbN)}, \ \forall w \in \ell^2(\bbN).
    \end{equation*}
    Let $\sigma: \ell^2(\bbN)^* \to \ell^2(\bbN)$ denote the mapping from an element in the dual space to its Riesz representation. Then $\sigma$ is a bijective isometric antilinear isomorphism; the Riesz representation can be constructed via an orthonormal basis through $\sigma(w^*) = \sum_{i \in I} w^{*}(e_i) e_i$, where $\set{e_i}_{i \in I}$ is some basis for $\calW$.

    Thus, the relation function on $\calX \times \calX$ that we need to approximate is
    \begin{equation*}
        r(x, y) = \iprod{\sigma \circ \Phi^* (y)}{\Phi(x)}_{\calW}, \ x, y \in \calX.
    \end{equation*}
    We do this by approximating $\Phi: \calX \to \calW$ with the MLP $\psi$ and approximating $\sigma \circ \Phi^*: \calX \to \calW$ with the MLP $\phi$.

    First, since $\Phi(x), \sigma \circ \Phi^*(y) \in \ell^2(\bbN), \forall x, y$, and $\calX$ is compact, we have
    \begin{equation*}
        \lim_{n \to \infty} \sup_{x,y \in \calX} \abs{r(x, y) - \sum_{i=1}^{n} (\Phi(x))_i \cdot (\sigma(\Phi^*(y)))_i} = 0.
    \end{equation*}
    Thus, let $d$ be such that,
    \begin{equation}\label{eq:thm1_proof_eq1}
        \sup_{x,y \in \calX} \abs{r(x, y) - \sum_{i=1}^{d} (\Phi(x))_i \cdot (\sigma(\Phi^*(y)))_i} < \frac{\varepsilon}{2}.
    \end{equation}
    We will obtain MLPs $\phi,\, \psi$ that are functions from $\calX$ to $\reals^{d}$. 
    By the universal approximation property of MLPs, for any $\tilde{\varepsilon} > 0$, there exists MLPs $\phi,\,\psi$ such that
    \begin{equation}\label{eq:thm1_proof_eq2}
        \sup_{x \in \calX} \abs{(\phi(x))_i - (\Phi(x))_{i}} < \tilde{\varepsilon} \ \text{ and } \ \sup_{x \in \calX} \abs{(\psi(y))_i - (\sigma(\Phi^*(x)))_{i}} < \tilde{\varepsilon}, \ \forall i \in \set{1, \ldots, d}.
    \end{equation}
    For example, \citet{cybenkoApproximationSuperpositions1989} shows that 1-layer neural networks with discriminatory activation functions of the form $\sum_{i=1}^{N} \alpha_i \sigma(w_j^\top x + b_j)$ are dense in the space of continuous functions.
    Now,
    \begin{align*}
        &\sup_{x,y \in \calX} \abs{r(x,y) - \hat{r}(x,y)} \\
        &= \sup_{x,y \in \calX} \abs{r(x,y) - \iprod{\phi(x)}{\psi(y)}} \\
        &\leq \sup_{x,y\in \calX} \paren{\abs{r(x,y) - \sum_{i=1}^{d} (\Phi(x))_i \cdot (\sigma(\Phi^*(y)))_i} + \abs{\sum_{i=1}^{d} (\Phi(x))_i \cdot (\sigma(\Phi^*(y)))_i - \iprod{\phi(x)}{\psi(y)}}}.
    \end{align*}
    The first term is less than $\varepsilon / 2$ by~\Cref{eq:thm1_proof_eq1}. Now, we bound the second term uniformly on $x,y \in \calX$, as
    \begin{align*}
        &\abs{\sum_{i=1}^{d} (\Phi(x))_i \cdot (\sigma(\Phi^*(y)))_i - \iprod{\phi(x)}{\psi(y)}} \\
        &\leq \sum_{i=1}^{d} \abs{(\Phi(x))_i \cdot (\sigma(\Phi^*(y)))_i - (\phi(x))_i(\psi_i(y))_i} \\
        &\leq \sum_{i=1}^{d} \paren{\abs{(\sigma(\Phi^*(y)))_i} \abs{(\sigma(\Phi^*(y)))_i - (\psi_i(y))_i} + \abs{(\Phi(x))_i} \abs{(\Phi(x))_i - (\phi(x))_i}}\\
        &\leq C(\Phi, \Phi^*, d) \sum_{i=1}^{d} \paren{\abs{(\sigma(\Phi^*(y)))_i - (\psi_i(y))_i} + \abs{(\Phi(x))_i - (\phi(x))_i}},
    \end{align*}
    where $C(\Phi, \Phi^*, d) \coloneq \max\sset{\max_{y \in \calX, i \in [d]} \aabs{(\sigma(\Phi^*(y)))_i}, \max_{x \in \calX, i \in [d]} \aabs{\abs{(\Phi(x))_i}}}$. Let the approximation error $\tilde{\varepsilon}$ in~\Cref{eq:thm1_proof_eq2} be small enough such that the above is smaller than $\varepsilon / 2$. This shows that
    \begin{equation*}
        \sup_{x,y \in \calX} \abs{r(x,y) - \tilde{r}_i(x,y)} \leq \frac{\varepsilon}{2} + \frac{\varepsilon}{2} = \varepsilon.
    \end{equation*}
\end{proof}

\begin{remark}
    The reason we assume that the underlying RKBS $\calB$ admits a feature map representation with feature space $\calW$ that is a Hilbert space is so that we can use the Riesz representation theorem. The Riesz representation theorem is what links the broad framework of reproducing kernel Banach spaces back to the inductive bias of modeling relations as inner products of feature maps.
\end{remark}

\begin{remark}
    \citet{zhangReproducingKernel2009} explore a specialization of reproducing kernel Banach spaces in which $\calB$ has a semi-inner product. This added structure grants semi-inner product RKBSs some desirable properties that RKHSs have but general RKBSs lack (e.g., convergence in the space implies pointwise convergence, weak universality of kernels, etc.). However, their notion of a semi-inner product is too restrictive to allow for our model $\iprod{\phi(x)}{\psi(x)}$.
\end{remark}

\section{Application: Analysis of Attention}\label{sec:app_attention}

In this section, we will use the results developed above to analyze the attention mechanism underlying the Transformer architecture that modern large language models are based on~\parencite[e.g.,][]{chungScalingInstructionFinetunedLanguage2022,openaiGPT4TechnicalReport2023,touvronLlamaOpenFoundation2023,taoriAlpacaStrongReplicable2023}.

Attention, popularized by the Transformer~\parencite{vaswani2017attention}, is a powerful mechanism for directing information flow in a machine learning model. Attention is the core component of Transformer models, which are composed of a sequence of ``blocks'' iterating between attention and a position-wise feedforward network. This simple architecture has achieved state-of-the-art results on a wide array of tasks, including natural language processing, vision, and speech recognition~\parencite[e.g.,][]{devlinBertPretrainingDeep2018,dongSpeechtransformerNorecurrenceSequencetosequence2018,dosovitskiyImageWorth16x162020,raffelExploringLimitsTransfer2020,liuSwinTransformerHierarchical2021}.

An attention module receives two inputs: a query $q \in \calX$ and a context $\bm{x} = (x_1, \ldots, x_n) \in \calX^n$. The aim of the attention module is to select the most relevant elements in the context based on the query. Formally, attention takes the form
\begin{equation}\label{eq:def_attention}
    \begin{split}
        &\mathrm{Attention}\paren{q, (x_1, \ldots, x_n)} = \sum_{i=1}^n \alpha_i x_i, \ \text{where,} \\
        &\alpha_i = \frac{\exp(\beta \iprod{\phi_\theta(q)}{\psi_\theta(x_i)})}{\sum_j \exp(\beta \iprod{\phi_\theta(q)}{\psi_\theta(x_j)})},
    \end{split}
\end{equation}
where $\theta$ are the parameters of the attention module and $\beta > 0$ is a scaling factor. Hence, attention assigns each element $x_i$ in the context $\bm{x}$ a weight $\alpha_i$, and returns a convex combination of all elements in the context. The weights $\{\alpha_i\}$ are computed using inner products between the query and each element in the context. These elements are normalized with the ``softmax'' function, $\mathrm{softmax}(\bm{z}) \coloneq \bbra{\exp(z_i) / \sum_j \exp(z_j)}_{i=1}^{n}$, so called because it models a smooth version of a maximization function.

Crucially, the inner products $\iprod{\phi_\theta(q)}{\psi_\theta(x_i)}$ in~\Cref{eq:def_attention} are asymmetric inner product relations of the form analyzed in~\Cref{sec:asymmetric_relations}. In particular, attention computes a relation between the query $q$ and each element in the context $x_i$, which captures the relevance of $x_i$ to $q$. Roughly speaking, an attention module retrieves the ``most relevant'' element in the context.

\begin{remark}
    In the standard implementation of attention in Transformers, $\phi_\theta$ and $\psi_\theta$ are linear maps. However, note that the preceding block ends with an MLP which processes all elements in the sequence in the same way. Hence, in Transformer attention, the inner product relations take the effective form $\iiprod{W_q \, \mathrm{MLP}(q)}{W_k \, \mathrm{MLP}(x_i)}$. A non-linear MLP followed by different linear projections has the same function class as two different non-linear MLPs (e.g., let $\mathrm{MLP} = (\phi_1, \ldots, \phi_d, \psi_1, \ldots, \psi_d)$ and $W_q, W_k$ the appropriate projection matrices). Hence, for ease of exposition, we analyze the case where $\phi_\theta, \psi_\theta$ are two different MLPs.
\end{remark}

In this section, we will use the universal approximation results for inner products of neural networks established in the previous sections to prove that attention can in fact always retrieve the ``most relevant'' element in any context using inner product relations.

We begin by formalizing the notion of ``most relevant element in a context.'' Fix a query $q \in \calX$. Let $\preceq_q$ be a query-dependent preorder relation on $\calX$ which specifies the relevance of different elements in $\calX$ to the query $q$. Formally, $\preceq_q$ is a complete (for each $a,b \in \calX$ either $a \preceq_q b$ or $b \preceq_q a$), reflexive ($a \preceq_q a$ for all $a \in \calX$), and transitive ($a \preceq_q b$ and $b \preceq_q c$ implies $a \preceq_q c$) relation.  We write $a \prec_q b$ if $a \preceq_q b$ and not $b \preceq_q a$, and we write $a \sim_q b$ if $a \preceq_q b$ and $b \preceq_q a$. Suppose each query $q \in \calX$ has a preorder $\preceq_q$ that defines a preordered space $(\calX, \preceq_q)$. The relation $x \prec_q y$ means that ``$y$ is more relevant to the query $q$ than $x$.''

For a given query $q$, the preordered space $(\calX, \preceq_q)$ defines the most relevant element with respect to the query $q$ among any context $\bm{x} = (x_1, \ldots, x_n) \in \calX^n$. We are interested in the ability of dot-product attention to retrieve the most relevant element for any given context. Hence, we are interested in approximating the function,
\begin{equation}\label{eq:def_select}
    \mathrm{Select}(q, (x_1, \ldots, x_n)) = \max\paren{\paren{x_1, \ldots, x_n}, \mathtt{key}=\preceq_q}.
\end{equation}
That is, $\mathrm{Select}(q, (x_1, \ldots, x_n))$ is the function that returns the most relevant element with respect to $q$. Formally, it returns $x_i$ when $x_i \succ_q x_j, \, \forall j \neq i$ (and may return an arbitrary element if no unique maximal element exists in the context).

We impose some natural regularity assumptions on the family of relevance preorders $\sset{\preceq_q}_{q \in \calX}$.
\begin{assumption}[Key-continuity]\label{ass:key_cts}
    Let $\calX$ be a compact Euclidean space. The family of relevance preorders $\sset{\preceq_q}_{q \in \calX}$ is said to be \textit{key-continuous} if, for each $q \in \calX$, the preorder $\preceq_q$ is continuous. That is, for all sequences $(x_i)_i$ such that $x_i \preceq_q y$ and $x_i \to x_\infty$, we have $x_\infty \preceq_q y$. Equivalently, for all $x \in \calX$, the sets $\sset{y \in \calX \colon y \preceq_q x}$ and $\sset{y \in \calX \colon x \preceq_q y}$ are closed with respect to the topology of $\calX$.
\end{assumption}

As the name suggests, the assumption of ``key-continuity'' captures a notion of continuity in the relevance preorder. We will borrow from utility theory in the economics literature to ultimately show that the ``relevance preorder'' can be represented by the inner product relations of attention.~\citet{debreuRepresentationPreferenceOrdering1954} derived necessary an sufficient conditions for the existence of a continuous ordinal utility function on a preordered topological space. \Citeauthor{debreuRepresentationPreferenceOrdering1954}'s representation theorems imply the following.

\begin{theorem*}[Existence of a continuous utility function for $\preceq_q$]
    Let $\calX$ be a compact Euclidean space. Suppose $\sset{\preceq_q}_{q \in \calX}$ satisfies the key-continuity assumption (\Cref{ass:key_cts}). Then, for each query $q \in \calX$, there exists a continuous function $u_q: \calX \to \reals$ such that
    \begin{equation*}
        x \preceq_q y \ \iff \ u_q(x) \leq u_q(y).
    \end{equation*}
\end{theorem*}
This follows by~\parencite{debreuRepresentationPreferenceOrdering1954}, whose most general result requires only the continuity of $\preceq_q$ and that $\calX$ is a second-countable space. We also refer the reader to~\parencite{jaffrayExistenceContinuousUtility1975} for an elementary proof.

We build on this to require an additional natural regularity condition on the family of relevance preorders $\sset{\preceq_q}_{q}$ that we call query-continuity.

\begin{assumption}[Query-continuity]\label{ass:query_cts}
    There exists a set of relevance utility functions $\set{u_q: \calX \to \reals}_{q \in \calX}$ such that $u(q, x) \coloneq u_q(x)$ is continuous in $q$.
\end{assumption}

Note that by~\Citeauthor{debreuRepresentationPreferenceOrdering1954}'s representation theorem, a relevance utility function $u_q: \calX \to \reals$ exists for each $q \in \calX$. This assumption further says that the relevance preorders $\sset{\preceq_q}_{q}$ are continuous in $q$ in the sense that the relevance of an element does not vary too much for two queries that are close to each other in $\calX$.

We are now ready to state our result, which says that inner product attention can approximate any selection function of the form in~\Cref{eq:def_select}, provided the family of relevance preorders $\sset{\preceq_q}_q$ satisfies the two natural regularity conditions stated above.

\begin{theorem}[Attention can approximate arbitrary information selection functions]
    Let $\sset{\preceq_q}_q$ be a family of relevance preorders that satisfies the query continuity and key continuity assumptions~\Cref{ass:key_cts,ass:query_cts}. Suppose that for all $\varepsilon > 0$, there exists $\eta_\epsilon > 0$ such that 
    $$\pprob{\min_{j \neq i} \aabs{u_q(x_i) - u_q(x_j)} > \eta_\epsilon} \geq 1 - \varepsilon.$$
    Then, there exist MLPs $\phi_\theta, \psi_\theta$ such that dot-product attention (\Cref{eq:def_attention}) approximates $\mathrm{Select}: \calX \times \calX^n \to \calX$ as defined in~\Cref{eq:def_select} arbitrarily well. Formally, for any $\epsilon > 0$, there exist MLPs $\phi_\theta, \psi_\theta$ such that with $\beta = \calO\paren{\eta_\epsilon^{-1} \log\paren{\epsilon^{-1} n \max_{x \in \calX} \norm{x}}}$, we have
    \begin{equation*}
        \norm{\mathrm{Attention}\paren{q, (x_1, \ldots, x)} - \mathrm{Select}\paren{q, (x_1, \ldots, x_n)}} < \epsilon
    \end{equation*}
    with probability at least $1 - \epsilon$.
\end{theorem}

\begin{proof}
    Let $r(x,y) \coloneq u_x(y)$ be the function to be approximated by the inner product relation $\hat{r}(x,y) \coloneq \iprod{\phi_\theta(x)}{\psi_\theta(y)}$ inside of attention. The scaling factor $\beta$ in~\Cref{eq:def_attention} will be determined later in the proof. By~\Cref{theorem:asymemtric_inner_prod_rel_func_class}, we have that for any $\epsilon_1 > 0$, there exist MLPs $\phi_\theta, \psi_\theta$ such that
    \begin{equation}\label{eq:attn_thm_proof_1}
        \sup_{x, y \in \calX} \abs{r(x,y) - \hat{r}(x, y)} < \epsilon_1.
    \end{equation}

    Consider the input query $q$ and context $(x_1, \ldots, x_n)$. Let $i = \argmax(u_q(x_1), \ldots, u_q(x_n))$ be the index of the most relevant element in the context. Observe that
    \begin{align*}
        \alpha_i &= \frac{\exp(\beta \, \hat{r}(q, x_i))}{\sum_{j=1}^{n} \exp(\beta \, \hat{r}(q, x_j))}\\
        &= \frac{1}{1 + \sum_{j \neq i} \exp(\beta \cdot (\hat{r}(q, x_j) - \hat{r}(q, x_i)))}
    \end{align*}

    By assumption, we have that with probability at least $1 - \epsilon$, $u_q(x_i) - u_q(x_j) > \eta_{\epsilon}$ for all $j \neq i$, where $\eta_{\epsilon} > 0$. Hence, under this event, the difference inside the exponential can be bounded by
    \begin{align*}
        \hat{r}(q, x_j) - \hat{r}(q, x_i) &\leq r(q, x_j) - r(q, x_i) + 2 \epsilon_1\\
        &= u_q(x_j) - u_q(x_i) + 2 \epsilon_1 \\
        &< - \eta_{\epsilon} + 2 \epsilon_1
    \end{align*}
    where the first line is by~\Cref{eq:attn_thm_proof_1}. Let $\epsilon_1 = \frac{1}{4}\eta_{\epsilon}$ so that $\hat{r}(q, x_j) - \hat{r}(q, x_i) < - \frac{1}{2} \eta_{\epsilon}$. Hence, we have that
    \begin{equation*}
        \beta \cdot (\hat{r}(q, x_j) - \hat{r}(q, x_i)) < - \frac{1}{2} \, \beta \, \eta_{\epsilon}.
    \end{equation*}

    This implies that
    \begin{align*}
        \alpha_i &> \frac{1}{1 + (n-1) \exp(- \frac{1}{2} \beta \eta_{\epsilon})} \\
        &= 1 - \frac{(n-1) \exp(- \frac{1}{2} \beta \eta_{\epsilon})}{1 + (n-1) \exp(- \frac{1}{2} \beta \eta_{\epsilon})} \\
        &> 1 - (n-1) \exp\paren{- \frac{1}{2} \beta \eta_{\epsilon}},
    \end{align*}
    where the last equality follows from the fact that $1 + (n-1) \exp(- \frac{1}{2} \beta \eta_{\epsilon}) > 1$.

    Let $\bar{\epsilon} \coloneq (n-1) \exp\paren{- \frac{1}{2} \beta \eta_{\epsilon}}$, which can be made arbitrarily small by choosing $\beta$ large enough. Then we have that
    \begin{align*}
        &\norm{\mathrm{Attention}\paren{q, (x_1, \ldots, x)} - \mathrm{Select}\paren{q, (x_1, \ldots, x_n)}} \\
        &= \norm{x_i - \sum_{j=1}^n{\alpha_j} x_j} \\
        &= \norm{(1 - \alpha_i) x_i - \sum_{j\neq i} {\alpha_j} x_j} \\
        &\leq (1 - \alpha_i)  \norm{x_i} + \sum_{j \neq i} \alpha_j \norm{x_j}\\
        &< \bar{\epsilon} \norm{x_i} + \bar{\epsilon}     \max_{j \neq i}\norm{x_j} \\
        &\leq 2 \bar{\epsilon} \max_{x \in \calX} \norm{x}.
    \end{align*}

    Note that $\max_{x \in \calX} \norm{x} < \infty$ since $\calX$ is compact by assumption. To achieve an error rate of $\epsilon$, let $\bar{\epsilon} = \frac{1}{2} \epsilon (\max_{x \in \calX} \norm{x})^{-1}$, which is achieved by $\beta \geq \frac{2}{\eta_\epsilon} \log\paren{\frac{2 (n-1) \max_{x \in \calX} \norm{x}}{\epsilon}}$. Hence, with probability at least $1 - \epsilon$,
    \begin{equation*}
        \norm{\mathrm{Attention}\paren{q, (x_1, \ldots, x)} - \mathrm{Select}\paren{q, (x_1, \ldots, x_n)}} < \epsilon.
    \end{equation*}
\end{proof}

\begin{remark}
    Note that the assumption that $\pprob{\min_{j \neq i} \aabs{u_q(x_i) - u_q(x_j)} > \eta_\epsilon} \geq 1 - \varepsilon$ simply says that the data distribution on $\calX \times \calX^n$ and the relevance preorder relations $\sset{\preceq_q}_q$ are such that there exists a ``most-relevant'' element in the context by a positive margin with high probability. This is a natural assumption since $\mathrm{Select}: \calX \times \calX^n \to \calX$ is only uniquely defined when a most-relevant element exists.
\end{remark}

\begin{remark}
    The same construction works for a variable-size context, $\calX^* \coloneq \cup_{k=1}^n \calX^k$. Note that $\beta$ only needs to scale logarithmically in the size of the maximal context size $n$. (Although, in general, $\eta_\epsilon$ also indirectly depends on $n$ through the data distribution.) Moreover, the scale of $\beta$ is not a restriction on the function class since it can be incorporated into the feature maps $\phi_\theta, \psi_\theta$.
\end{remark}

\section{Discussion}

Our analysis underscores the importance of kernels for learning relations and attention mechanisms. In the symmetric case, the assumption of a positive-definite kernel function is natural, leading to the standard framework of reproducing kernel Hilbert spaces. In the asymmetric case, which is arguably more important and applicable for relational learning, a different technical approach is needed, and reproducing kernel Banach spaces arises naturally. After completing this work we became aware of the related work of \citet{wright2021transformers}, which makes this connection as well.

The results presented here can be extended in several ways. For example, the bounds on the number of neurons in a perceptron that suffice to approximate a relation function to a given accuracy can likely be sharpened, drawing on the extensive literature on approximation properties of neural networks \citep[e.g.,][]{petrushev1998approximation,pinkus1999approximation,makovoz1998uniform,burger2001error,maiorov2006approximation,bachBreakingCurseDimensionality2016}. In terms of attention mechanisms in transformers, our initial focus was on approximating the most relevant key to a given query. The representation theorem of \citet{debreuRepresentationPreferenceOrdering1954} is used to express the problem in terms of a utility function, which is then approximated. It would be of interest to derive approximation bounds for the full distribution of attention values that are computed by the softmax function in Transformers. Finally, when considering relational learning, the possibility of higher-order, recursive relations, naturally arises \citep[e.g.,][]{altabaaRelationalConvolutionalNetworks2023}, and it may be interesting to study function spaces of hierarchical relations in such settings.

\newpage
\section*{Acknowledgment}
This work is supported by the funds provided by the National Science Foundation and by DoD OUSD (R\&E) under Cooperative Agreement PHY-2229929 (The NSF AI Institute for Artificial and Natural Intelligence).

\printbibliography

\newpage
\listoffixmes

\end{document}